\documentclass{article}

\usepackage{times}
\usepackage{soul}
\usepackage{url}
\usepackage[hidelinks]{hyperref}
\usepackage[utf8]{inputenc}
\usepackage[small]{caption}
\usepackage{graphicx}
\usepackage{amsmath}
\usepackage{amsthm}
\usepackage{booktabs}
\usepackage{algorithm}
\usepackage{algorithmic}
\newtheorem{example}{Example}

\usepackage{latexsym}
\usepackage{amsmath}
\usepackage{amssymb}
\usepackage{amstext}
\usepackage{mathtools}
\usepackage{mathrsfs} 
\usepackage{ifsym}
\usepackage{bbm} 
\usepackage{color}
\usepackage[svgnames]{xcolor}
\usepackage{graphics}
\usepackage{graphicx}
\usepackage{tikz} 
\usetikzlibrary{shapes.misc,automata,positioning,arrows}
\usepgflibrary{shapes.arrows} 
\usepackage{dsfont}
\usepackage{url}
\usepackage{setspace}
\usepackage{varioref}
\usepackage{xspace} 
\usepackage{stmaryrd}
\usepackage{engord}
\usepackage{comment}
\usepackage{pifont}
\usepackage{amsmath}

\newcommand{\ie}{\mbox{i.e.}}
\newcommand{\eg}{\mbox{e.g.}}
\newcommand{\cf}{\mbox{cf.}}

\newcommand{\wrt}{\mbox{w.r.t.}}
\renewcommand{\st}{\mbox{s.t.}}

\newcommand{\defined}{\ensuremath{\equiv_{\text{def}}}}

\newtheorem{definition}{Definition}
\newtheorem{proposition}{Proposition}
\newtheorem{corollary}{Corollary}

\newcommand{\KB}{\ensuremath{\mathcal{K}}} 

\newcommand{\limp}{\rightarrow} 

\newcommand{\sat}{\Vdash}

\newcommand{\I}{\ensuremath{\mathcal{I}}}

\newcommand{\A}{\ensuremath{\mathcal{A}}}
\newcommand{\B}{\ensuremath{\mathcal{B}}}

\newcommand{\C}[2]{C_{#1}(#2)}  

\usepackage{xspace}
\usepackage{bussproofs}

\newcommand{\lang}{\ensuremath{\mathcal{L}}}

\newcommand{\cond}{\ensuremath{\Rightarrow}}
\renewcommand{\C}{\ensuremath{\mathcal{C}}}

\newcommand{\clent}{\vDash }

\newcommand{\Sleq}{\ensuremath{\preceq}}

\newcommand{\nd}{\noindent}
\newcommand{\ii}[1]{\mbox{$(#1)$}}

\newif\ifmycomment
\mycommentfalse 

\newcommand{\mycomment}[1]{\vspace*{3mm}{ 
\ifmycomment
\fbox{
\begin{minipage}{0.875\columnwidth}
\scriptsize {\bf Comment}: #1 
\end{minipage}}
\vspace{3ex}
\fi
}}

\title{A General Framework for Modelling Conditional Reasoning - \\ Preliminary Report
}

\author{Giovanni Casini$^{1,2}$ \and Umberto Straccia$^1$}
\date{%
    $^1$ISTI - CNR, Pisa, Italy\\%
    $^2$CAIR, University of Cape Town, Cape town, South Africa\\[2ex]%
    \{giovanni.casini, umberto.straccia\}@isti.cnr.it
    }

\begin{document}

\maketitle

\begin{abstract}
We introduce and investigate here a formalisation for conditionals that allows the definition of a broad class of  reasoning systems. This framework covers the most popular kinds of conditional reasoning  in logic-based KR:  the semantics we propose is appropriate for a structural analysis of those conditionals that do not satisfy closure properties associated to classical logics. 

\end{abstract}

\section{Introduction}\label{sect_intro}

Conditionals are generally considered the backbone of human (and AI) reasoning: the ``if-then" connection between two propositions is the stepping stone of  arguments and a lot of the research effort in formal logic has focused on this kind of connection. 
A conditional connection satisfies different properties according to the kind of arguments it is used for. The classical material implication is appropriate for modelling the ``if-then" connection as it is used in Mathematics, but the equivalence between the material implication $A\limp B$ and $\neg A\lor B$ is not appropriate for many other contexts. 
 Different kinds of reasoning use different kinds of conditionals, modelling, among others, \emph{presumptive reasoning} (\eg~``Birds typically fly"), \emph{normative reasoning} (\eg~``if you have had alcohol, you should not drive"), \emph{casual reasoning} (\eg~``if you throw a stone against that window, then you will break it"), \emph{probabilistic reasoning} (\eg~``if you go out in this weather, you will probably get a cold"), \emph{fuzzy reasoning} (\eg~``if the temperature is hot, then the fan speed is high"), or \emph{counterfactual reasoning} (\eg~``if I were you, I wouldn't do that").

In different contexts we associate to the ``if-then" expressions distinct modalities, each of them validating different argumentation patterns. A common way of formalising different reasoning patterns that are or are not endorsed in a specific reasoning context is through \emph{structural properties}. That is, formal constraints specifying that a set of conditionals is closed under certain reasoning patterns. 
This kind of analysis was used already in classical logic, as the class of Tarskian logical consequence relations have been characterised in terms of three main properties:

\begin{description}
\item[Reflexivity:] $A\clent A \vspace*{2em} \hspace{0.3cm} \text{(Ref)} $

\vspace{-0.4cm}
\item[Monotonicity:]
{\Large$
\frac{ A\clent C,  \ \clent B\rightarrow A}{ B\clent C }$} \hspace{0.3cm} \text{(Mon)}

\item[Cut:]
{\Large$
\frac{ A\land B\clent C,  \ A\clent B}{ A\clent C}$} \hspace{0.3cm} \text{(Cut)} \ .

\end{description}

%

    
%
\nd Referring to structural properties in analysing conditional logics has become a standard in some areas~\cite{Gabbay1995,Makinson1994,Makinson2000}. However, let us note that while some properties may appear obvious in everyday reasoning, these may become in fact undesirable depending on the reasoning context in which we apply them. For example, a property like 
\begin{description}
\item[Right Conjunction:]
{\Large$
\frac{ A\cond B,  \ A\cond C }{ A\cond (B\land C)}$} \hspace{0.3cm} \text{(And)}

\end{description}
%
%
\nd dictates that if an agent believes ``if $A$ then $B$" and ``if $A$ then $C$", then it should also believe that  ``if $A$ then $B$ and $C$" ($\cond$ stands for conditional implication). For instance, if an agent believes that presumably birds fly and that presumably winter days are cold, it is reasonable to require for a rational agent to abide to the (And) property, and, thus, to believe the conjunction of the two, \ie~presumably birds fly \emph{and} winter days are cold.

While the (And) property is required in presumptive reasoning, it is not considered appropriate for other kinds of reasoning, as, for example, in a probabilistic context~\cite{HawthorneMakinson2007} or in deontic reasoning. In the latter case, in some kind of normative reasoning involving incompatible preferences, (And) is not a desirable reasoning pattern: an agent could believe $\mathtt{saturday}\cond \mathtt{party}$ (``On Saturday night  I would like to go to a party'') and $\mathtt{saturday}\cond \mathtt{tv}$ (``On Saturday night  I would like to stay home watching TV''), but not $\mathtt{saturday}\cond \mathtt{party}\land \mathtt{tv}$ (``On Saturday night  I would like to go to a party and to stay home watching TV'').
 
 Another property that is usually satisfied in most of the reasoning contexts is


\begin{description}
\item[Right Weakening:]
{\Large$
\frac{ A\cond B,  \ \clent B\rightarrow C }{ A\cond C }$} \hspace{0.3cm} \text{(RW)} \ .
\end{description}

%
\nd (RW) simply states that if an agent believes ``if $A$ then $B$", then it believes also ``if $A$ then $C$" for any (classical) consequence $C$ of $B$. For example, it is reasonable to impose that believing that presumably birds fly implies also believing that presumably birds move, since flying implies moving. 

However, there are contexts in which (RW) gives back counter-intuitive results, as in some forms of deontic and causal reasoning \cite{CasiniEtAl2019}, as illustrated by the following examples:
\vspace{-0.1cm}
\begin{itemize}
    \item ``if you are involved in a car accident, you should remain on the spot'' is an acceptable norm, but ``if you are involved in a car accident, you should remain on the spot or paint yourself in blue'' is not as acceptable;

\vspace{-0.1cm}
    \item ``if you turn the wheel of a moving car, the car will move in a circle'' is meaningful, while  ``If you turn the wheel of a moving car, the car will move'' is not really that meaningful;

\vspace{-0.1cm}
    \item ``if you throw a stone against the window, it will break'' is meaningful, but ``If you throw a stone against the window, it will break or Ann will drink tea'' is not.
\end{itemize}


\vspace{-0.1cm}
\nd (RW) is a property that is strongly connected to the traditional semantics that is used to formalise conditional reasoning, \ie~possible-worlds semantics. In fact, most formalisations of conditional reasoning have been built using a possible-worlds semantics by referring more or less directly to classical modal operators. 
Using such an  approach it has been possible to define logical systems modelling various kinds of non-classical reasoning.

On the other hand, relying on possible worlds means relying on closed logical theories, 
and such an approach enforces some properties (\eg\ logical omniscience) that may be in conflict with some modelling goals. Some works have already considered ways of combining a possible world approach with some constrained forms of (RW)~\cite{CasiniEtAl2019,Rott89}. Let us anticipated that, in contrast to those approaches, we will consider here a kind of intentional semantics instead.

One limit of the possible-worlds approach to the formalisation of conditionals 
``if \emph{condition} $C$ holds, then \emph{effect} $D$ holds with a given \emph{modality}" 
is that it accounts for the modality that is associated with the truth of~$D$ given the truth of~$C$. However, it does not account for whether the truth of~$D$ given the truth of~$C$ has any relevance for the kind of reasoning we are considering. 
The centrality of the notion of \emph{relevance} in conditional reasoning has already been pointed out in~\cite{Delgrande2011}. However what `relevance' means in the context of conditional reasoning remains still vague nowadays. 

As we are going to show in the next section, our formalisation focuses on choice functions that model what the \emph{agent considers as relevant effects and relevant conditions}. Our work is somewhat inspired by~\cite{Rott2001} that also suggested the use of choice functions in modelling the semantics of conditionals.

The paper is organised as follows. In the next section we introduce some background concepts we will rely on in our formalisation of conditionals. In Section~\ref{sect_sem} we illustrate our formalisation of conditionals, while Section~\ref{sect_structural} describes how we may accommodate various structural properties within our approach. Section \ref{sect_ent} discusses how to formalise entailment relations in our framework and shows possible future developments. Eventually, Section~\ref{concl} summarises our contribution.



\section{Preliminaries}\label{sect_pre}

\nd We use a  conditional language containing  conditionals of the form $C\cond D$. We do not consider here the possibility of nesting the conditionals or combining them via propositional operators. 

Let  $\lang$ be a finitely generated propositional language, with logical connectives $\neg,\lor,\land,\limp$ and $\leftrightarrow$ and propositional symbol $\bot$ having usual meaning. Capital letters $A,B,\ldots$ will be used to refer to propositions, while $\A,\B,\ldots$ will refer to sets of propositions. 
%
With $\clent$  we denote the classical propositional consequence relation. 

Our language will be $\lang_{\cond}$, the conditional language built on top of $\lang$: namely,
\[
\lang_{\cond}\defined\{C\cond D\mid C,D\in\lang\} \ .
\]

\nd On the semantics side we will use a relation
$\leq\subseteq\lang\times\lang$ among propositional formulae, where $A\leq B$ iff $\clent A\limp B$, so that $\leq$ generates the classical lattice semantics over propositional formulas, with $\lor$ and $\land$ represented by the \emph{join} and \emph{meet} operations, respectively. 
The relations $<$ and $\equiv$ are defined as usual from $\leq$. Note that, using $\leq$ as a representation of $\limp$, $A<B$ represents $A\limp B$ and $\neg (B\limp A)$, while $A\equiv B$ is a representative of $A\leftrightarrow B$. Of course, $\leq$ is reflexive and transitive.



With $\min_{\leq}(\A)$ we denote the \emph{minimal} elements in $\A$ \wrt~$\leq$, \ie~
$
\min_{\leq}(\A) \defined \{ B \in \A | \not \!\exists C\in \A~\st~C < B\}$, while $\A^{\uparrow}  \defined  \{B\mid A\leq B\text{ for some }A\in\A\}$ and $
 \A^{\downarrow}  \defined  \{B\mid B\leq A\text{ for some }A\in\A\}$ (we will write $A^{\uparrow}, A^{\downarrow}$  for $\{A\}^{\uparrow}, \{A\}^{\downarrow}$).



\nd We are going to use a well-known order among sets of formulae, based on $\leq$: the \emph{Smyth} order $\Sleq$ over power sets (see, \eg~\cite[Section 3]{Straccia09} for a short introduction).\footnote{Orders of this type are often used in the context of so-called \emph{power domains}~\cite{Knijnenburg93,knijnenburg96,Plotkin76,Smyth78,Winskel85}.}
Specifically, 
\[
\A \Sleq \B \text{ iff } \forall B\in\B \ \exists A\in\A \ \st~ A\leq B \ .
\]
\nd We also write $\A \cong \B$ iff $\A \Sleq \B$ and $\B \Sleq \A$.

A \emph{choice} function is a set-valued function 
$h\colon\lang \to 2^{\lang} $,
 mapping a formula to a set of formulae. We say that $h$ is \emph{Smyth-monotone}, or simply \emph{S-monotone}, 
 iff for every $A,B\in\lang$, if $ A\leq B$ then $h(A)\Sleq h(B)$. Furthermore, $A\in\lang$ is a \emph{fixed-point} of $h$ iff $A \in h(A)$ (see, \eg~\cite{Straccia09}). 

Eventually, we say that  $h$ is \emph{$\star$-closed}, 
where $\star \in \{ \leq, \equiv \}$, iff for all $A,B,C \in \lang$, if $A \in h(C)$ and $B\star A$ then $B \in h(C)$.\footnote{Note that for $\leq$ order matters as $\leq$ is not symmetric.}
On the other hand, we will say that  $h$ is \emph{$\star$-closed},  where $\star \in \{ \land, \lor \}$, 
iff for all $A,B,C \in \lang$, if $A \in h(C)$ and $B \in h(C)$ then $A\star B \in h(C)$.



\section{Semantics}\label{sect_sem}

We build our semantics on top of two choice functions, $f$ and $g$, representing what an agent considers as relevant connections. 
Specifically, a \emph{conditional interpretation} $\I$ is a pair 

\vspace{-0.2cm}
\[
\I=(f,g)
\]

\vspace{-0.2cm}
\nd \st~$f: \lang \to 2^{\lang}$ and $g: \lang \to 2^{\lang}$. $f$ represents the \emph{relevant effects} of a proposition, and $g$ the \emph{possible conditions} for a proposition to hold. 



\begin{definition}[Satisfaction]\label{def_satisfaction}
Let $\I=(f,g)$ be a conditional interpretation. $\I$ \emph{satisfies} a conditional $A\cond B$, denoted $\I\sat A\cond B$, iff the following conditions hold:
\begin{enumerate}
    \item there is $B' \in \lang$ s.t. $B'\in f(A)$ and $B'\leq B$; and
    \item $A\in g(B)$.
\end{enumerate}

\nd $A\cond B$ is \emph{satisfiable} (has a \emph{model}) if there is a conditional interpretation $\I$ such that $\I\sat A\cond B$. A set of conditionals is \emph{satisfiable} (has a \emph{model}) iff each conditional in it is so.
\end{definition}

\nd Fig. \ref{fig_1} gives a graphical representation of the satisfaction relation: $\I\sat A\cond B$ iff there is a ``triangle'' $A \xrightarrow{f}B'\leq B\xrightarrow{g} A$. We indicate with $A\triangle B$ that  there is a triangle $A \xrightarrow{f}B'\leq B\xrightarrow{g} A$ passing through some $B'\leq B$.

The meaning of the above definition has an epistemic flavour: an agent accepts a conditional connection between $A$ and $B$ if $B$ is a logical consequence of some \emph{relevant effect} $B'$ of $A$ ($B'\in f(A)$), and $A$ is recognised as a \emph{relevant condition} for $B$ ($A\in g(B)$).



%
\begin{figure}[tb]
  \centering
    \includegraphics[width=0.4\textwidth]{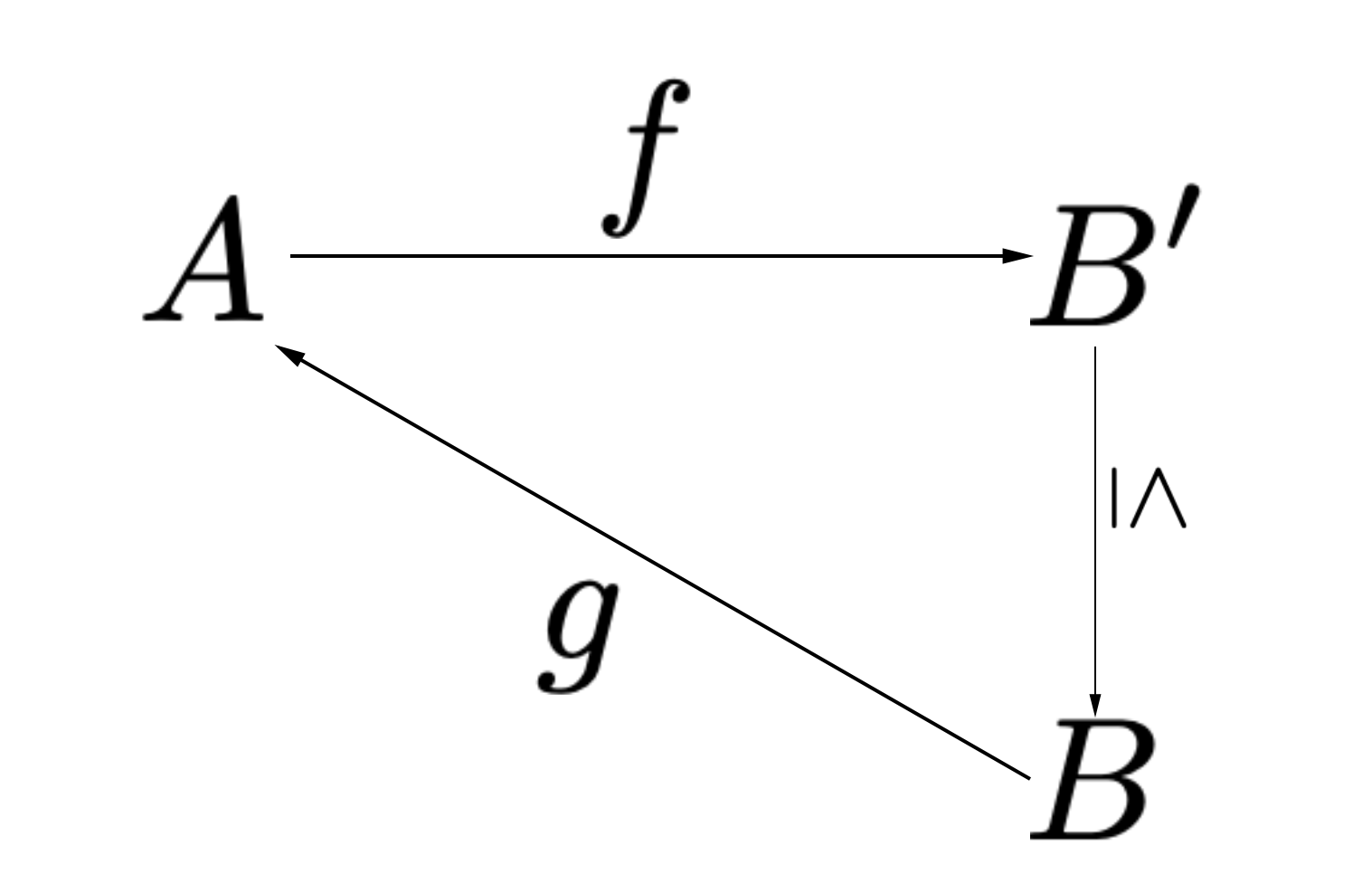}
    \caption{Graphical representation of $\I \sat A\cond B$.}\label{fig_1}
\end{figure}

Given an interpretation $\I$, with $S_{\I}$ we indicate the set of conditionals satisfied by $\I$, \ie~$S_{\I}\defined\{A\cond B\mid \I\sat A \cond B\}$.

Let us note that our class of interpretations is quite generic and, in particular, can represent any set of conditionals. In fact, given a set of conditionals $S$, we may define a model $\I$ characterising it, that is, satisfying exactly the conditionals in $S$ $(\ie, S_{\I}=S$). To do so, given $S$, we construct a conditional interpretation \wrt~$S$
\[
\I_S=(f_S,g_S)    
\]
 
\nd in the following way: 
\begin{enumerate}
    \item define the following sets:    
    $\A_{B}   \defined  \{A\mid A\cond B\in S\}$ and 
    $\C_{A}  \defined  \{B\mid A\cond B\in S\}$. 
    

    \item for every $D\in\lang$, we set
    \[
       f_S(D) = \min_{\leq}(\C_{D}) \text{\ \ and \ \ }  g_S(D) =\A_{D} \ .
    \]
\end{enumerate}

\nd $\I_S$ characterises $S$, as the following proposition proves.
\begin{proposition}\label{prop_characteristicmodel}
Given a set of conditionals $S$, $\I_S$ is its \emph{characteristic model}, that is, 
a conditional $A\cond B$ is in $S$ iff $\I_S\sat A\cond B$.
\end{proposition}

\begin{proof}
    From left to right. Assume $A\cond B$ is in $S$. Then, by definition of $\I_S$, we have that $A\in g_S(B)$ and there is a $B'\in f_S(A)$ s.t. $B'\leq B$ (it could be $B$ itself). Hence $\I_S\sat A\cond B$.
    
    From right to left. Assume $\I_S\sat A\cond B$. Then $A\in g_S(B)$, and that, by construction of $\I_S$, can be only if $A\cond B\in S$.  
\end{proof}


\nd Please note that, as we have proved Proposition \ref{prop_characteristicmodel} for any arbitrary set of conditionals $S$, the following immediate corollary tells us that the class of conditional interpretations $\I=(f,g)$ do not impose any form of closure under any structural property.

\begin{corollary}\label{coroll_noclosure}
The class of conditional interpretations can represent any set of conditionals.
\end{corollary}


\begin{corollary}\label{coroll_sat}
Any set of conditionals $S$ is satisfiable.
\end{corollary}

\section{Structural Properties}\label{sect_structural}

In the following, we are going to show that by constraining the functions $f$ and $g$, it is possible to enforce the closure of the set of conditionals under structural properties that are considered as appropriate for modelling various kinds of reasoning. We start by analysing some classical reasoning patterns. 

\mycomment{Umberto: provide some intuition on semantic condition for each structural property, if possible}


At first, as $f$ and $g$ range over formulae and not over possible worlds, \ie~logically closed theories, \emph{Definition \ref{def_satisfaction} does not imply any form of closure under logical equivalence}. Such a behaviour may be desirable in some epistemic contexts in which we would like to avoid some form of a priori \emph{logical omniscience}~\cite{FaginEtAl1995}. That is, the well known reasoning patterns of  
\emph{Left Logical Equivalence} (LLE) and \emph{Right Logical Equivalence} (RLE) do not hold in general in our framework. However, if these are desired, it is quite straightforward to enforce (LLE) and (RLE) in our setting. Specifically, for


    

\begin{description}
\item[Left Logical Equivalence:]
\begin{eqnarray*}
\frac{ A\cond C,  \hspace{0.3cm} A \equiv B}{ B\cond C } & \text{(LLE)}
\end{eqnarray*}
\end{description}

\nd it suffices to impose the following semantic constraints on a conditional interpretation $\I=(f,g)$: 

%

\begin{description}
\item[(LLE$_\I$)] for all $A,B$:
\begin{enumerate}
    \item if $A\equiv B$, then $f(A)=f(B)$;
    \item $g$ is $\equiv$-closed.
\end{enumerate}
\end{description}
\nd Similarly, for 

\begin{description}
\item[Right Logical Equivalence:]
\begin{eqnarray*}
\frac{ A\cond B,  \hspace{0.3cm} B \equiv C}{ A\cond C } & \text{(RLE)}
\end{eqnarray*}
\end{description}

\nd the semantic constraint to be imposed on a conditional interpretation $\I=(f,g)$ is: 

\begin{description}
\item[(RLE$_\I$)] for all $A,B$:
\begin{enumerate}
    \item if $A\equiv B$, then $g(A)=g(B)$.
\end{enumerate}
\end{description}

\nd The conditions (LLE$_\I$) and (RLE$_\I$) characterise the classes of the conditional interpretations satisfying, respectively, (LLE) and (RLE). In fact, it can be shown that\footnote{Since the proof is straightforward we  omit it.}

\begin{proposition}\label{prop_le}
A set of conditionals $S$ is closed under (LLE) (resp.~(RLE)) iff it can be characterised by a conditional model $\I=(f,g)$ that satisfies (LLE$_{\I}$) (resp.~(RLE$_{\I}$)).
\end{proposition}


\nd Another basic property is \emph{Reflexivity}, that simply states that for every proposition $A$ it holds `If $A$, then $A$'. Despite appearing as an obviously valid conditional, there are some contexts in which it is not a desirable property. Consider for example a deontic system expressing recommendations, in which $A\cond B$ is read as ``if $A$ holds, then $B$ would be preferrable''. This kind of conditionals can result quite counter-intuitive if it embeds reflexivity (see \eg~\cite{Makinson2000}): while  ``if there is an act of violence, then you should call the police'' appears to be a reasonable conditional, to be forced to conclude ``if there is an act of violence, then there should be an act of violence'' is counter-intuitive. Reflexivity does not hold in our framework, though if we would like to have this pattern, it suffices to impose a simple constraint on conditional interpretations. For
\begin{description}
\item[Reflexivity:]
\begin{eqnarray*}
A\cond A & \text{(Ref)}
\end{eqnarray*}
\end{description}
\nd the semantic constraint to be imposed on a conditional interpretation $\I=(f,g)$ is: %
\begin{description}
\item[(Ref$_\I$)] for all $A$:
\begin{enumerate}
    \item $A$ is a fixed-point of both $f$ and $g$.
\end{enumerate}
\end{description}






\begin{proposition}\label{prop_ref}
A set of conditionals $S$ is closed under (Ref) iff it can be characterised by a conditional model $\I=(f,g)$ that satisfies (Ref$_{\I}$).
\end{proposition}

\begin{proof}
From right to left. Assume that $S$ is characterised by some conditional model $\I=(f,g)$ that satisfies (Ref$_{\I}$), that is, $S=S_\I$. We have to show that $S_\I$ is closed under (Ref), and it is immediate to see that (Ref$_{\I}$) implies $\I\sat A\cond A$ for every $A$.

From left to right. Let $S$ be a set of conditionals closed under (Ref). We need to prove that there is a conditional interpretation $\I=(f,g)$ characterising it and satisfying (Ref$_{\I}$). We can define such an $\I$ by slightly modifying the characteristic model $\I_S=(f_S,g_S)$. Specifically, it suffices to consider $\I=(f,g)$, where  $g(A)=g_S(A)$ and $f(A)=f_S(A)\cup\{A\}$, for every $A$. Clearly, $\I$ satisfies (Ref$_{\I}$).  The proof that $A\cond B\in S$ iff $\I\sat A\cond B$ is analogous to the proof of Proposition~\ref{prop_characteristicmodel}, considering also that $S$ is closed under (Ref).
\end{proof}

\nd As next, we consider more elaborate structural properties. We start with considering the \emph{Cut} reasoning pattern, one of the main structural properties in classical logic (\cf~Section~\ref{sect_intro}). So, for

\begin{description}
\item[Cut:]
\begin{eqnarray*}
\frac{ A\land B\cond C,  \hspace{0.3cm} A\cond B }{ A\cond C } & & \text{(Cut)}
\end{eqnarray*}
\end{description}

\nd the semantic constraints to be imposed on a conditional interpretation $\I=(f,g)$ are: %
\begin{description}
\item[(Cut$_\I$)] for all $A,B,C$:
\begin{enumerate}
    \item If $A\triangle B$, then $f(A)\Sleq f(A\land B)$;
    \item If $A\in g(B)$ and $A\land B\in g(C)$, then $A\in g(C)$.
\end{enumerate}
\end{description}

\nd Then, we can show that

\begin{proposition}\label{prop_cut}
A set of conditionals $S$ is closed under (Cut) iff it can be characterised by a conditional model $\I=(f,g)$ that satisfies (Cut$_{\I}$).
\end{proposition}

\begin{proof}
From right to left. Assume that $S$ is characterised by some conditional model $\I=(f,g)$ that satisfies
(Cut$_{\I}$), that is $S=S_\I$. We need to prove that $S_\I$ is closed under (Cut). Suppose $\I\sat A\cond B$ and $\I\sat A\land B\cond C$. Then there is some $C'\in f(A\land B)$ s.t. $C'\leq C$. Since $f(A)\Sleq f(A\land B)$ there is some $C''\in f(A)$ s.t. $C''\leq C'$, that is, $C''\leq C$. Regarding $g$, we have $A\in g(B)$ and $A\land B\in g(C)$, hence $A\in g(C)$. $C''\in f(A)$ and $A\in g(C)$ imply $\I\sat A\cond C$. Therefore,  $S_\I$ is closed under (Cut).

From left to right. Let $S$ be a set of conditionals closed under (Cut). We need to prove that there is a conditional interpretation $\I=(f,g)$ characterising it and satisfying (Cut$_{\I}$). Let us consider the characteristic model $\I_S$.
We need to prove that it satisfies the two conditions of (Cut$_{\I}$). So, assume $A\triangle B$ and $C\in f_S(A\land B)$. $A\triangle B$ implies $\I_S\sat A\cond B$, that by Proposition \ref{prop_characteristicmodel} implies $A\cond B\in S$. By construction of $\I_S$, $C\in f_S(A\land B)$ implies that $A\land B\cond C\in S$. From $\{A\cond B, A\land B\cond C\}\subseteq S$ and (Cut) we have that $A\cond C\in S$, and, by Proposition \ref{prop_characteristicmodel}, $\I_S\sat A\cond C$. Therefore, there is a $C'$ s.t. $C'\in f_S(A)$ and $C'\leq C$. That is, $f_S(A)\leq_S f_S(A\land B)$ holds.  Regarding the second condition on $g_S$, let $A\in g_S(B)$ and $A\land B\in g_S(C)$. By construction of $\I_S$, $A\cond B$ and $A\land B\cond C$ are in $S$, and by (Cut) $A\cond C\in S$. Therefore, by construction of $\I_S$, $A\in g_S(C)$, which concludes the proof.
\end{proof}

\nd As next, we address \emph{monotonicity} (\cf~Section~\ref{sect_intro}), also a main property of classical logic.  It states that strengthening the antecedent of a conditional from a logical point of view, we still preserve the effects. 
For example, the conditional $\mathtt{horse}\cond \mathtt{mammal}$ in a monotonic system imposes to conclude that any kind of horse is a mammal, \eg~$\mathtt{horse}\land \mathtt{mustang}\cond \mathtt{mammal}$. That is, (Mon) makes our conditionals  \emph{strict}, in the sense that they do not admit exceptions. So, for

\begin{description}
\item[Monotonicity:]
\begin{eqnarray*}
\frac{ A\cond C,  \hspace{0.3cm} \clent B\limp A 
} { B\cond C } & & \text{(Mon)} 
\end{eqnarray*}
\end{description}

\nd the semantic constraints to be imposed on a conditional interpretation $\I=(f,g)$ are: %
\begin{description}
\item[(Mon$_{\I}$)] \
\begin{enumerate}
    \item $f$ is S-Monotone;
    \item $g$ is $\leq$-closed.
\end{enumerate}
\end{description}

\begin{proposition}\label{prop_mon}
A set of conditionals $S$ is closed under (Mon) iff it can be characterised by a conditional model $\I=(f,g)$ that satisfies (Mon$_{\I}$).
\end{proposition}

\begin{proof}
From right to left.  Assume that $S$ is characterised by some conditional model $\I=(f,g)$ that satisfies  (Mon$_{\I}$), that is, $S=S_\I$. We need to prove that $S_\I$ is closed under (Mon). Assume $\I\sat A\cond C$ and $\clent B\limp A$, \ie~$B \leq A$. $\I\sat A\cond C$ implies that there is some $B'\in f(A)$ s.t. $B'\leq C$. By S-Monotonicity $f(B) \Sleq f(A) $ and, thus, there is some $B''\in f(B)$ s.t. $B''\leq B'$, that implies $B''\leq C$. As $g$ is $\leq$-closed and $B \leq A$, $A\in g(C)$ implies $B \in g(C)$. Hence $\I\sat B \cond C$. Therefore,  $S_\I$ is closed under (Mon).

From left to right. Let $S$ be a set of conditionals closed under (Mon). We need to prove that there is a conditional interpretation $\I=(f,g)$ characterising it and satisfying (Mon$_{\I}$). Let us consider the characteristic model $\I_S$ of $S$ as by Proposition~\ref{prop_characteristicmodel}. We need to prove that it satisfies the two conditions of (Mon$_{\I}$). So, let $B\leq A$, and let $C\in f_S(A)$. By construction of $\I_S$, $C\in f_S(A)$ implies that $A\cond C\in S$, and by (Mon) $B\cond C\in S$. By construction of $\I_S$, either $C\in f_S(B)$, or there is a $C'\in f_S(B)$ s.t. $C'\leq C$. Hence $f_S$ is S-Monotone. Also the 
$\leq$-closure of $g_S$ is an immediate consequence of the closure under (Mon) of $S$ and the definition of $g_S$ in $\I_S$, , which concludes the proof.
\end{proof}

\nd As by Section~\ref{sect_intro}, (And) is a property that appears desirable in many contexts, but may have some exceptions. For 

\begin{description}
\item[Right Conjunction:]
\begin{eqnarray*}
\frac{ A\cond B,  \hspace{0.3cm} A\cond C }{ A\cond (B\land C) } & & \text{(And)} 
\end{eqnarray*}
\end{description}

\nd the semantic constraints to be imposed on a conditional interpretation $\I=(f,g)$ that characterise (And) are: %
\begin{description}
\item[(And$_{\I}$)] for all $A,B$:
\begin{enumerate}
    \item if $B,C\in \min_{\leq}(f(A))$, then $B\equiv C$;
    \item $g(A)\cap g(B)\subseteq g(B\land C)$.
\end{enumerate}
\end{description}

\begin{proposition}\label{prop_and}
A set of conditionals $S$ is closed under (And) iff it can be characterised by a conditional model $\I=(f,g)$ that satisfies (And$_{\I}$).
\end{proposition}

\begin{proof}
From right to left. Assume that $S$ is characterised by some conditional model $\I=(f,g)$ that satisfies
 (And$_{\I}$), that is, $S=S_\I$. We need to prove that $S_\I$ is closed under (And). Assume $\I\sat A\cond B$ and $\I\sat A\cond C$. Then there is some $B'\in f(A)$ \st~$B'\leq B$ and some $C'\in f(A)$ \st~$C'\leq C$. $\min_{\leq}(f(A))$ contains some $B^*$ s.t. $B^*\leq B$ and some $C^*$ s.t. $C^*\leq C$. By the first condition of  (And$_{\I}$) we have $B^*\equiv C^*$, and as a consequence we have  $B^*\leq C$ and eventually $B^*\leq B\land C$. 
Regarding $g$, we have $A\in g(B)$ and $A\in g(C)$, hence $A\in g(B\land C)$. $B^*\in f(A)$, $B^*\leq B\land C$ and $A\in g(B\land C)$ together imply $\I\sat A\cond (B\land C)$. Therefore, $S_\I$ is closed under (And).

From left to right. Let $S$ be a set of conditionals closed under (And). We need to prove that there is a conditional interpretation $\I=(f,g)$ characterising it and satisfying (And$_{\I}$). Let us consider the characteristic model $\I_S$ of $S$ as by Proposition~\ref{prop_characteristicmodel}. We need to prove that it satisfies the two conditions of (And$_{\I}$). So, assume there are three propositions  $A,B,C$ \st~$B,C\in \min_{\leq}(f_S(A))$ and $B\not\equiv C$. From the construction of $\I_S$ we have that $B,C\in \min_{\leq}(f_S(A))$ implies that $B\land C\notin f_S(A)$, and that for any $A,B$, if $B\in f_S(A)$, then $\I_S\sat A\cond B$. Hence we have $\I_S\sat A\cond B$ and $\I_S\sat A\cond C$, but not $\I_S\sat A\cond B\land C$, against the closure of $S$ under (And). Regarding the second condition, for all $A,B$,  $A\in g_S(B)$ iff $\I_S\sat A\cond B$. Let $A\in g_S(B)\cap g_S(C)$. Then $\I_S\sat A\cond B$, $\I_S\sat A\cond C$, and, by (And), $\I_S\sat A\cond B\land C$, that implies $A\in g_S(B\land C)$, which concludes the proof.
\end{proof}

\nd Reasoning by cases is another well-known characteristics of classical reasoning, which is formalised by the \emph{Left Disjunction} reasoning pattern. To deal with it, for

\begin{description}
\item[Left Disjunction:]
\begin{eqnarray*}
\frac{ A\cond C,  \hspace{0.3cm} ~B\cond C }{ A\lor B\cond C } & & \text{(Or)} 
\end{eqnarray*}
\end{description}

\nd the semantic constraints to be imposed on a conditional interpretation $\I=(f,g)$ that characterise (Or) are: %
\begin{description}
\item[(Or$_{\I}$)] for all $A,B$:
\begin{enumerate}
    \item $\min_{\leq}(f(A)^{\uparrow} \cap f^{\uparrow}(B))\subseteq f(A\lor B)$;
    \item $g$ is $\lor$-closed.
\end{enumerate}
\end{description}







\begin{proposition}\label{prop_or}
A set of conditionals $S$ is closed under (Or) iff it can be characterised by a conditional model $\I=(f,g)$ that satisfies (Or$_{\I}$).
\end{proposition}

\begin{proof}
From right to left. Assume that $S$ is characterised by a conditional model $\I=(f,g)$ that satisfies
(Or$_{\I}$), that is, $S=S_\I$. We need to prove that $S_{\I}$ is closed under (Or). Assume $\I\sat A\cond C$ and $\I\sat B\cond C$. Then there are $C'\in f(A)$ \st~$C'\leq C$, and $C''\in f(B)$ \st~$C''\leq C$. Then there must be some $C^*$ s.t. $C'\leq C^*$, $C''\leq C^*$, and $C^*\leq C$ ($C$ itself satisfies the constraint), and the minimal  among them w.r.t. $\leq$ are in $f(A\lor B)$ by condition~1.~of (Or$_{\I}$). Hence in $f(A\lor B)$ there is some $C^*$ s.t. $C^*\leq C$. But, $\I\sat A\cond C$ and $\I\sat B\cond C$ imply also that  $A,B\in g(C)$ and, thus, as $g$ is $\lor$-closed, $A\lor B\in g(C)$. Therefore, we can conclude $\I\sat A\lor B\cond C$. Therefore, $S_{\I}$ is closed under (Or).

From left to right. Let $S$ be a set of conditionals closed under (Or). We need to prove that there is a conditional interpretation $\I=(f,g)$ characterising it and satisfying (Or$_{\I}$). So, let us consider the characteristic model $\I_S$ as by Proposition~\ref{prop_characteristicmodel}. At first, we show that $\I_S$ satisfies the second condition of (Or$_{\I}$). In fact, by construction of $\I_S$, for all $C,D$, if $C\in g_S(D)$ then $C\cond D\in S$. Therefore, as $S$ is closed under (Or), $g_S$ must be $\lor$-closed.
On the other hand, if $\I_S$ does not satisfy the first condition of (Or$_{\I}$), we transform $\I_S$ into a model $\I'$ by extending $f_S$ only. Specifically, it is sufficient that for every disjunction $A\lor B$ we add  the set $\min_{\leq}(f_S)^{\uparrow}(A)\cap f_S^{\uparrow}(B))$ to $f_S(A\lor B)$. Now, it is easily verified that indeed $\I'$ satisfies exactly the same set of conditionals as $\I_S$, \ie~$S$. In fact, in $\I'$ we have an extension of $f_S$, while $g_S$ stays the same. Therefore, as by construction of $\I_S$,  $C\in g(D)$ iff $C\cond D\in S$, the same holds for $\I'$ and, thus, the set of satisfied conditionals by $\I'$ remains the same as for $\I_S$, \ie~$S$, which concludes the proof.
\end{proof}


\nd As mentioned in Section~\ref{sect_intro}, \emph{Right Weakening} is a property that is generally desirable in many context with some exceptions. To support the reasoning pattern of

\begin{description}
\item[Right Weakening:]
\begin{eqnarray*}
\frac{ A\cond B,  \hspace{0.3cm} \clent B\rightarrow C 
}{ A\cond C } & & \text{(RW)} 
\end{eqnarray*}
\end{description}

\nd the semantic constraint to be imposed on a conditional interpretation $\I=(f,g)$ that characterise (RW) is: %
\begin{description}
\item[(RW$_{\I}$)] for all $A,B$:
\begin{enumerate}
    \item if $A\leq B$ then $g(A)\subseteq g(B)$.
\end{enumerate}
\end{description}

%



\begin{proposition}\label{prop_rw}
A set of conditionals $S$ is closed under (RW) iff it can be characterised by a conditional model $\I=(f,g)$ that satisfies (RW$_{\I}$).
\end{proposition}

\begin{proof}
From right to left. Assume that $S$ is characterised by a conditional model $\I=(f,g)$ that satisfies (RW$_{\I}$), that is, $S=S_\I$. We need to prove that $S_{\I}$ is closed under (RW). So, assume $\I\sat A\cond B$ and $\clent B\limp C$, \ie~$B \leq C$. Then there is some $B'\in f(A) $ \st~$B'\leq B$, and consequently $B'\leq C$. Since   $B \leq C$, $A\in g(B)$, by condition 1. we have $A\in g(C)$. Hence $\I\sat A\cond C$. $S_{\I}$ is closed under (RW).

From left to right. Let $S$ be a set of conditionals closed under (RW). We need to prove that there is a conditional interpretation $\I=(f,g)$ characterising it and satisfying (RW$_{\I}$). So, consider the characteristic model $\I_S$, assume $B\leq C$, \ie~ $\clent B\limp C$, and let $A\in g_S(B)$. By construction of $\I_S$, $A\in g_S(B)$ implies $\I\sat A\cond B$, that, by (RW), implies $\I\sat A\cond C$. By construction of $\I_S$, $A\in g_S(C)$, as desired. 
\end{proof}

\nd So far, we have taken under consideration most of the properties characterising classical entailment. However, we still miss two important consistency properties: namely,  \emph{ex falso quodlibet} and \emph{consistency preservation}. The former is a classical property strongly connected with classical implication and entailment, and stating that we can conclude anything from a false premise. This property, for example, is not fully desirable in counterfactual reasoning, where we would like to be able to reason coherently about false situation, but that are at least conceivable. 
Nevertheless,  to support the reasoning pattern of

\begin{description}
\item[Ex Falso Quodlibet:]
\begin{eqnarray*}
\frac{ \clent \neg A }{ A\cond B} & & \text{(EFQ)} 
\end{eqnarray*}
\end{description}

\nd the semantic constraints to be imposed on a conditional interpretation $\I=(f,g)$ that characterise (EFQ) are: %
\begin{description}
\item[(EFQ$_{\I}$)] for all $A$: if $A\equiv \bot$, then
\begin{enumerate}
    \item $\bot\in f(A)$;
    \item $A\in g(B)$, for all $B$.
\end{enumerate}
\end{description}





\begin{proposition}\label{prop_efq}
A set of conditionals $S$ is closed under (EFQ) iff it can be characterised by a conditional model $\I=(f,g)$ that satisfies (EFQ$_{\I}$).
\end{proposition}

\begin{proof}
From right to left. Assume that $S$ is characterised by a conditional model $\I=(f,g)$ that satisfies (EFQ$_{\I}$), that is, $S=S_\I$. We need to prove that $S_{\I}$ is closed under (EFQ). Assume $\clent\neg A$. We need to prove that $\I\sat A\cond B$ holds for all $B$. $\clent\neg A$ implies $A\equiv \bot$, hence, by (EFQ$_\I$), we have $\bot\in f(A)$, $\bot\leq B$ and, thus, $A\in g(B)$. Therefore, $\I\sat A\cond B$ and, thus, $S_{\I}$ is closed under (EFQ).

From left to right. Let $S$ be a set of conditionals closed under (EFQ). We need to prove that there is a conditional interpretation $\I=(f,g)$ characterising it and satisfying (EFQ$_{\I}$). So, consider the characteristic model $\I_S$ and let $\clent \neg A$. By (EFQ), $A\cond \bot\in S$ follows, and, since  $\bot\in \min_{\leq}(\C_{D})$, $\bot\in f_S(A)$ holds. Furthermore, by (EFQ), $A\cond B\in S$ holds, for all $B\in\lang$. Therefore, by construction of $\I_S$, $A\in g_S(B)$ holds, for any $B\in\lang$ and, thus, $\I_S$ satisfies (EFQ$_{\I}$), which concludes the proof.
\end{proof}

\nd Please note that (EFQ) is an immediate consequence of (RLE), (And) and (RW). However, we may have contexts  that do not satisfy some of these three properties, but still satisfies (EFQ). If this is the case, the semantic constraint (EFQ$_{\I}$) has to be considered.

Consistency preservation tells us that we cannot conclude absurdity from a classically consistent premise.
To support the reasoning pattern of

\begin{description}
\item[Consistency Preservation:]
\begin{eqnarray*}
\frac{ A\cond B,  \hspace{0.3cm} \clent \neg B }{  \clent \neg A } & & \text{(Con)} 
\end{eqnarray*}
\end{description}


\nd the semantic constraint to be imposed on a conditional interpretation $\I=(f,g)$ that characterise (Con) is: %
\begin{description}
\item[(Con$_{\I}$)] for all $A$, 
\begin{enumerate}
    \item if $B \in f(A)$, for some $B \leq \bot$,  then $A \leq \bot$.
\end{enumerate}
\end{description}

\nd Please note that only if we assume (RLE) we can express (Con) in the classical (equivalent) forms  

\vspace{-0.3cm}
\begin{eqnarray*}
\frac{ A\cond \bot }{  \clent \neg A } & & \frac{ \not\clent \neg A }{ A\not\cond \bot }\\ 
\end{eqnarray*}

\vspace{-0.3cm}


\nd where the reading of the latter is: ``if $\neg A$ is not a tautology then the conditional $A\cond \bot$ cannot be concluded". 

\begin{proposition}\label{prop_con}
A set of conditionals $S$ is closed under (Con) iff it can be characterised by a conditional model $\I=(f,g)$ that satisfies (Con$_{\I}$).
\end{proposition}

\begin{proof}
From right to left. Assume that $S$ is characterised by a conditional model $\I=(f,g)$ that satisfies (Con$_{\I}$), that is $S=S_\I$.  We need to prove that $S$ is closed under (Con). So, assume $\I\sat A\cond B$ and $\clent \neg B$, \ie~$B\leq\bot$. We need to prove that $\clent \neg A$ holds. 
$\I\sat A\cond B$ implies that there is $B' \in f(A)$ \st~$B' \leq B$, hence $B' \leq \bot$. Therefore,
by (Con$_{\I}$) we have $A \leq \bot$, that is, $\clent \neg A$.

From left to right. Let $S$ be a set of conditionals satisfying (Con). We need to prove that there is a conditional interpretation $\I=(f,g)$ characterising it and satisfying (Con$_{\I}$). We prove that the characteristic model $\I_S$ is such an interpretation, by proving that for any $A$, if $\bot<A$ then there is no $B\leq\bot$ s.t. $B\in f_S(A)$. Let $\bot< A$ and $B\leq \bot$; hence $\not\clent\neg A$ and $\clent\neg B$. By (Con), $A\cond B\notin S$. By construction of $\I_S$, $A\cond B\notin S$ implies that $B$ is not in $f_S(A)$, since otherwise we would have $A\cond B\in S$.

\end{proof}

\nd A stronger property that connects conditional reasoning to classical entailment is \emph{supraclassicality}, that is, the conditional systems extends classical reasoning.
To support the reasoning pattern of

\begin{description}
\item[Supraclassicality:]
\begin{eqnarray*}
\frac{ \clent A\limp B }{  A\cond B } & & \text{(Sup)} 
\end{eqnarray*}
\end{description}

\nd the semantic constraints to be imposed on a conditional interpretation $\I=(f,g)$ that characterise (Sup) are: %
\begin{description}
\item[(Con$_{\I}$)] for all $A$, 
\begin{enumerate}
    \item $A$ is a fixed-point of $f$;
    \item $A^{\downarrow}\subseteq g(A)$.
\end{enumerate}
\end{description}






\begin{proposition}\label{prop_sup}
A set of conditionals $S$ is closed under (Sup) iff it can be characterised by a conditional model $\I=(f,g)$ that satisfies (Sup$_{\I}$).
\end{proposition}

\begin{proof}
From right to left. Assume that $S$ is characterised by a conditional model $\I=(f,g)$ and, thus, $S= S_\I$, that satisfies (Sup$_{\I}$). We need to prove that $S_{\I}$ is closed under (Sup). So, assume $\clent A\limp B$, \ie~$A \leq B$. Then $A\in B^{\downarrow}$, hence $A\in g(B)$, $A\in f(A)$, and $A\leq B$, hence $\I\sat A\cond B$.

From left to right. Let $S$ be a set of conditionals satisfying (Sup). We need to prove that there is a conditional interpretation $\I=(f,g)$ characterising it and satisfying (Sup$_{\I}$). Consider the characteristic model $\I_S$: it clearly satisfies the second condition, the one over $g$. It is possible it does not satisfy the condition over $f$, in case $S$ contains some conditional $A\cond B$ with $B< A$. To cover such a case it is sufficient to modify $\I_S$ into a model $\I$ in the same way as done in the proof of Proposition~\ref{prop_ref}. $\I$ is a characteristic model of $S$ satisfying both the conditions in (Sup$_{\I}$).
\end{proof}

\nd Please note that \ii{i} (Sup) is a consequence of (Ref) and (RW) together, but it is not equivalent to the combination of those two properties; and
\ii{ii} if we change the second condition in (Sup$_\I$) into $A^{\downarrow}= g(A)$, we model the classical propositional entailment (proof omitted).

A main portion of the research in conditional reasoning has focused on forms of defeasible reasoning. Defeasible reasoning is characterised by a degree of uncertainty connected some of the drawn conclusions that may be revised when faced with more complete and specific information. Presumptive reasoning, that is, reasoning based on expectations, represents the most popular context in which it is necessary to constraint (Mon). The basic form of constrainted monotonicity is \emph{Cautious Monotonicity}. 
To support the reasoning pattern of

\begin{description}
\item[Cautious Monotonicity:]
\begin{eqnarray*}
\frac{ A\cond B,  \hspace{0.3cm} A\cond C }{ A\land B\cond C } & & \text{(CM)} 
\end{eqnarray*}
\end{description}

\nd the semantic constraints to be imposed on a conditional interpretation $\I=(f,g)$ that characterise (CM) are: %
\begin{description}
\item[(CM$_{\I}$)] for all $A,B$, 
\begin{enumerate}
    \item if $A\triangle B$, then $f(A\land B)\Sleq f(A)$;
    \item if $A\in g(B)\cap g(C)$ then $A\land B\in g(C)$.
\end{enumerate}
\end{description}

%





\begin{proposition}\label{prop_cm}
A set of conditionals $S$ is closed under (CM) iff it can be characterised by a conditional model $\I=(f,g)$ that satisfies (CM$_{\I}$).
\end{proposition}

\begin{proof}
From right to left. Assume that $S$ can be characterised by a conditional model $\I=(f,g)$ and, thus, $S= S_\I$, that satisfies (CM$_{\I}$). We need to prove that $S_{\I}$ is closed under (CM). So, assume $\I\sat A\cond B$ and $\I\sat A\cond C$. Therefore, there is some $C'\in f(A)$ s.t. $C'\leq C$, and $f(A\land B)\leq_S f(A)$. As a consequence, there is some $C''\in f(A \land B)$ s.t. $C''\leq C'$, that is, $C''\leq C$. Regarding $g$, we have $A\in g(B)$ and $A\in g(C)$, hence $A\land B\in g(C)$. Therefore, $\I\sat A\land B\cond C$ holds.

From left to right. Let $S$ be a set of conditionals closed under (CM). We need to prove that there is a conditional interpretation $\I=(f,g)$ characterising it and satisfying (CM$_{\I}$). Consider the characteristic model $\I_S$ as by Proposition \ref{prop_characteristicmodel}. We need to prove that it satisfies the two conditions of (CM$_{\I}$). Let $A\triangle B$ and $C\in f_S(A)$. $A\triangle B$ implies $\I_S\sat A\cond B$, which by Proposition \ref{prop_characteristicmodel} implies $A\cond B\in S$. By construction of $\I_S$, $C\in f_S(A)$ implies that $A\cond C\in S$. From $\{A\cond B, A\cond C\}\subseteq S$ and (CM) we have that $A\land B\cond C\in S$, and, by Proposition~\ref{prop_characteristicmodel}, $\I_S\sat A\land B\cond C$. 
That is, there is a $C'$ s.t. $C'\in f_S(A\land B)$ and $C'\leq C$. Therefore, $f_S(A\land B)\Sleq f_S(A)$ holds.  Regarding the second condition on $g_S$, let $A\in g_S(B)\cap g_S(C)$. By construction of $\I_S$, $A\cond B$ and $A\cond C$ are in $S$, and by (CM) $A\land B\cond C\in S$, that is, by construction of $\I_S$, $A\land B\in g_s(C)$, which concludes the prove.
\end{proof}

\nd Beyond being a desirable property from the point of view of many reasoning contexts, such as presumptive and prototypical reasoning \cite{KrausEtAl1990}, (CM) if formally important because combining it with (Cut) we obtain \emph{Cumulativity}:

\begin{description}
\item[Cumulativity:]
\begin{eqnarray*}
\text{If }A\cond B\text{ then } (A\cond C\text{ iff }A\land B\cond C ) & \text{(Cumul)}
\end{eqnarray*}
\end{description}

\nd (Cumul) is formally important because entailment relations satisfying (Cumul) satisfy also \emph{Idempotence}, a classical  closure property. 

\mycomment{lo posso lasciare cosí, o  dovrei dedicare più spazio a spiegare cumulativit\'a ed idempotenza, secondo te? SI}


\nd The semantic constraints to be imposed on a conditional interpretation $\I=(f,g)$ that characterise (Cumul) are obtained by combining (Cut$_\I$) and (CM$_\I$): that is,

\begin{description}
\item[(Cumul$_\I$)] for all $A,B,C$, 
\begin{enumerate}
    \item If $A\triangle B$ then $f(A)\cong f(A\land B)$;
    \item If $A\in g(B)$ then ( $A\in g(C)$ iff $A\land B\in g(C)$ ).
\end{enumerate}
\end{description}

\nd Proceeding in this way we can introduce many other structural properties / reasoning patterns as formal constraints specified over the the functions $f$ and $g$. For example, consider (AntiRW), a form of constrained (RW)~\cite{CasiniEtAl2019}:

\begin{description}
\item[Anti Right Weakening:]
\begin{eqnarray*}
\frac{ A\cond B,  \hspace{0.3cm} \clent B\limp C,  \hspace{0.3cm} \clent C\limp D,  \hspace{0.3cm}  A\not\cond C } { A\not\cond D } & & \text{(AntiRW)} 
\end{eqnarray*}
\end{description}

\nd Or, equivalently,

\begin{eqnarray*}
\frac{ A\cond B,  \hspace{0.3cm} \clent B\limp C,  \hspace{0.3cm} \clent C\limp D,  \hspace{0.3cm}  A\cond D } { A\cond C } & & \text{(AntiRW*)} 
\end{eqnarray*}

\nd (AntiRW), that is implied by (RW), states that we can weaken the conclusions, but, once we \emph{block} the right weakening process, we cannot recover it anymore. It is a property that, for example, appears appropriate for some causal or deontic forms of reasoning (see~\cite{CasiniEtAl2019} for  details). 

We can enforce (AntiRW) in our framework via the following semantic constraints:

\begin{description}
\item[(AntiRW$_\I$)] for all $A,B,C,D$, 
\begin{enumerate}
    \item if $A\in g(B)$, $A\in g(D)$ and $B\leq D$ then  $B\leq C\leq D$ implies $A\in g(C)$ .
\end{enumerate}
\end{description}



\begin{proposition}\label{prop_antirw}
A set of conditionals $S$ is closed under (AntiRW) iff it can be characterised by a conditional model $\I=(f,g)$ that satisfies (AntiRW$_{\I}$).
\end{proposition}

\begin{proof}
From right to left.  Assume that $S$ can be characterised by a conditional model $\I=(f,g)$ and, thus, $S= S_\I$, that satisfies (AntiRW$_{\I}$). We prove that $S_{\I}$ is closed under (AntiRW*) (that is equivalent to (AntiRW)). So, assume $\I\sat A\cond B$ and $\I\sat A\cond D$, with $B \leq C \leq D$. Then there is some $B'\in f(A)$ s.t. $B'\leq B\leq C\leq D$. Also, $A\in g(B)$ and $A\in g(D)$, that, by condition 1. of (AntiRW$_{\I}$), imply $A\in g(C)$. The latter and $B'\leq C$ imply $\I\sat A\cond C$, as desired.

From left to right. Let $S$ be a set of conditionals closed under (AntiRW*). We need to prove that there is a conditional interpretation $\I=(f,g)$ characterising it and satisfying (AntiRW$_{\I}$). Let us consider the characteristic model $\I_S$, and we prove that it satisfies the condition (AntiRW$_{\I}$). So, 
let $A\in g_S(B)$, $A\in g_S(D)$, and $B\leq C\leq D$. By the construction of $\I_S$ we have $\I_S\sat A\cond B$ and $\I_S\sat A\cond D$. Since $B\leq C\leq D$ and $S$ is closed under (AntiRW*), $\I_S\sat A\cond C$, that implies $A\in g_S(C)$. Hence condition 1. is satisfied, which completes the prove.
\end{proof}

\nd Finally, let $\mathcal{P}$ be the set of structural properties presented in this section.  We have taken under consideration each of them, and we have given a semantic counterpart in our framework. Each semantic property is a sufficient condition for obtaining a characterising model, but not a necessary condition. Specifically, given any set of conditionals $S$ closed under some structural property (X), we have proved that there must be a characterising model satisfying (X$_\I$), not that every model characterising $S$ must satisfy (X$_\I$).

In the following, we clarify whether all these semantic properties are compatible among them. That is, given a set of conditionals closed under some of the  structural properties in $\mathcal{P}$, we are going to answer to the problem whether there is a characterising model closed under all the correspondent semantic properties.

\begin{proposition}
Let $\mathcal{X}\subseteq\mathcal{P}$ be a set of structural properties in $\mathcal{P}$, and  $\mathcal{X}_\I$ be the set of the correspondent semantic properties. If a set $S$  of conditionals is closed under the properties in $\mathcal{X}$, then there is a conditional interpretation characterising $S$ and satisfying all the properties in $\mathcal{X}_\I$.
\end{proposition}

\begin{proof}
(Sketch) Let $\mathcal{P'}=\mathcal{P}\setminus\{\text{(Ref),(Sup),(Or)}\}$. If $\mathcal{X}\subseteq \mathcal{P'}$ the proof is straightforward: as seen in the proof of the propositions in this section, given a set $S$ satisfying any property in $\mathcal{P'}$, the characteristic model $\I_S$ satisfies the correspondent semantic property. So, if we are dealing only with properties in $\mathcal{P'}$, the characteristic model of $S$ is the model we are looking for. It remains to take under consideration the combinations between properties in $\mathcal{P'}$ and $\{\text{(Ref),(Sup),(Or)}\}$. 

For (Ref) in Proposition \ref{prop_ref} we have extended $f_S$ in the model $\I_S$ into a function $f$ s.t. $f(A)=f_S(A)\cup\{A\}$ for every proposition $A$, and that the new model satisfies the same set of conditionals $S$. It is easy to check that the satisfaction of any property in $\mathcal{P}$ and of their semantic counterparts is preserved in this extension of $f$, with the only exception of (LLE), that requires a further extension of $f$: namely, for every $A$, $f(A)=f_S(A)\cup\{B\mid B\equiv A\}$. It is easy to check that, given any set of conditionals closed under (LLE) and (Ref), this further extension of $f$ \wrt~$f_S$ does not affect neither the set of conditionals satisfied by the model (that is, it is still the characteristic model of the initial set $S$), nor the satisfaction of the other semantic properties. 

For (Sup) we introduce the same extension to $f_S$, and the same argument applies.

For (Or) in Proposition \ref{prop_or}, we define a model $\I$ that extends $\I_S$ by adding $\min_{\leq}(f_S(A)^{\uparrow} \cap f_S(B)^{\uparrow}$ to $f_S(A\lor B)$, for every disjunction $A\lor B$. Again, this change of $\I_S$ does not affect any of the other semantic properties, apart from (LLE), that requires an extra change as for (Ref) and (Sup): we need to extend $f_S$ imposing $f(C)=f_S(C)\cup\min_{\leq}(f_S(A)^{\uparrow} \cap f_S(B)^{\uparrow})$ to any $C$ s.t. $C\equiv (A\lor B)$ for some disjunction $A\lor B$. As for (Ref) and (Sup), this extra change does not affect the set of the satisfied conditionals and the satisfaction of the other semantic properties, which completes the prove.
\end{proof}

\section{Entailment and Future Work}\label{sect_ent}

\mycomment{Umberto: say something more that entailment can be checked also via reasoning patterns/structural properties allowed}

\mycomment{Future work: also address other reasoning patterns/structural properties such as Rational Monotonicity (RM)}

Most of the results in this paper are \emph{representational} ones showing how conditional interpretations are appropriate for modelling different forms of closure. The next step is the definition of an actual reasoning systems in this framework: we start from a finite set of conditionals $\KB=\{A_1 \cond B_1, \ldots, A_n \cond B_n\}$, and we would like to derive new conditionals according to reasoning patterns satisfied, or, more generally, according to some predefined functions $f$ and $g$. In this preliminary report, we  present only intuition behind our approach that aims at modelling conditionals entailed by predefined functions $f$ and $g$.

To do so, we consider the following example for illustrative purposes, showing how one may derive new conditionals, for instance, under (Ref) and (Cut). 

\begin{example}\label{ex_reasoning}
Let  $\KB=\{\mathtt{feline}\cond \mathtt{carnivore},\ \mathtt{feline}\land \mathtt{carnivore}\cond \mathtt{mammal}\}$ (we  use only the initials of the propositional letters in what follows). The conditionals in $\KB$ represent the information an agent is aware of. That is, if $A\cond B\in \KB$ then the agent is aware that $B$ is a relevant effect of $A$ and $A$ is a relevant condition for $B$. Formally, this translate into a model $\I=(f,g)$ where, for every $A$,
\begin{eqnarray*}
f(A) & \defined & \{B\mid A\cond B\in \KB\} \\
g(A) & \defined & \{B\mid B\cond A\in \KB\} \ .
\end{eqnarray*}

\nd Hence in the present case we have $f(\mathtt{f})=\{\mathtt{c}\}$, $f(\mathtt{f}\land \mathtt{c})=\{\mathtt{m}\}$ and $f(A)=\emptyset$ for any other formula $A$; $g(\mathtt{c})=\{\mathtt{f}\}$, $g(\mathtt{m})=\{\mathtt{f}\land \mathtt{c}\}$ and $g(A)=\emptyset$ for any other formula $A$. This model satisfies only the conditionals in $\KB$, and in order to impose the closure under (Ref) and (Cut), we impose the satisfaction of (Ref$_\I$) and (Cut$_\I$) by extending $f$ and $g$ into, respectively, $f'$ and $g'$: in order to satisfy (Ref$_\I$) we add $A$ to $f(A)$ and $g(A)$ for every formula $A$, while to satisfy (Cut$_\I$) we need to add $\mathtt{m}$ to $f(\mathtt{f})$ (for condition 1.) and $\mathtt{f}$ to $g(\mathtt{m})$ (for condition 2.). Hence, we end up with the model $\I'=(f',g')$ with $f'(\mathtt{f})=\{\mathtt{c},\mathtt{m},\mathtt{f}\}$, $f(\mathtt{f}\land \mathtt{c})=\{\mathtt{m}, \mathtt{f}\land \mathtt{c}\}$ and $f(A)=\{A\}$ for any other formula $A$; $g(\mathtt{c})=\{\mathtt{f},\mathtt{c}\}$, $g(\mathtt{m})=\{\mathtt{f}, \mathtt{f}\land \mathtt{c}, \mathtt{m}\}$ and $g(A)=\{A\}$ for any other formula $A$. To determine which conditionals are satisfied by $\I'$, we have to look for `triangles' (see Fig.~\ref{fig_1}) that occur under $f'$, $g'$ and $\leq$. 
In this case, one may verify that indeed $\I'$ satisfies also $\mathtt{f}\cond \mathtt{m}$, \ie, 
$\I' \sat \mathtt{f}\cond \mathtt{m}$ (``a feline is a mammal") and all the reflexive conditionals.
\end{example}

\nd Therefore, the main idea to formalise reasoning is, given a knowledge base $\KB$, to build a model characterising $\KB$ and then to modify its $f$ and $g$ according to the reasoning patterns we would like to implement. The first step is to define closure operations over $f$ and $g$ that result into the smallest extension of $\KB$ satisfying the desired properties, in line with classical Tarskian approach to entailment. This is the approach taken in Example~\ref{ex_reasoning}, that is compatible with the structural proprieties we have considered here: all of them can be used also as derivation rules, and are compatible with the existence of a single smallest closure.

The following step would be the definition of forms of reasoning that are stronger from the inferential point of view, looking at more complex structural properties that allow for multiple smallest closed extensions. This would be in line with some popular approaches for modelling defeasible reasoning  using possible-worlds semantics: they take under consideration more complex structural properties like  \emph{Rational Monotonicity}, and define the entailment relations referring to specific semantic constructions~\cite{LehmannMagidor1992,Lehmann1995,Pearl1990,CasiniEtAl2019b}.

Beyond the development of decision procedures built on top of this semantics, we would also like to point out the flexibility of our approach. In particular working on the variation of two aspects: 
\ii{i} the configuration of the satisfaction relation; and 
\ii{ii} the interpretation of the relation $\leq$. 

\begin{figure}[tb]
  \centering
    \includegraphics[width=0.4\textwidth]{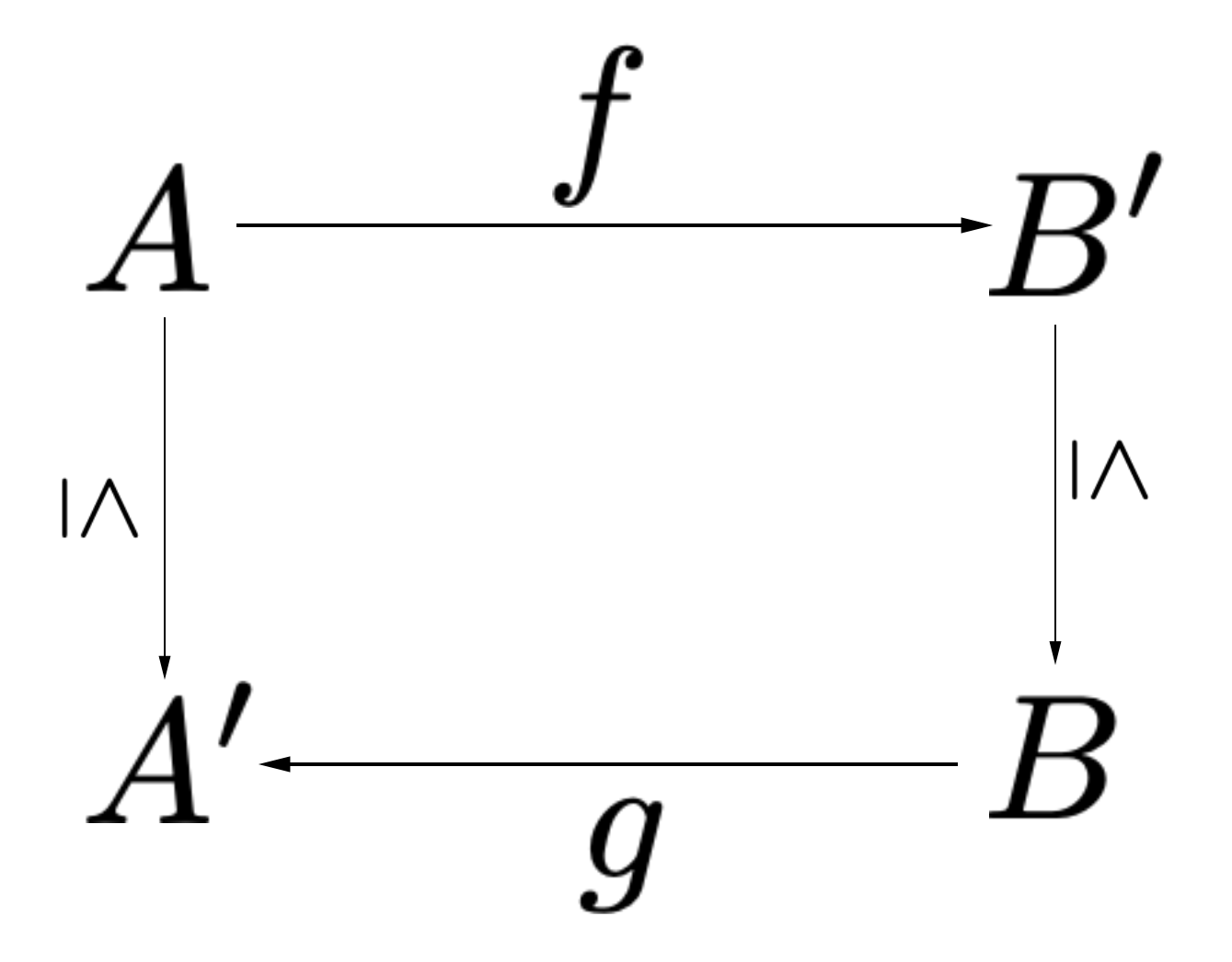}
    \caption{Alternative configuration for  $\I \sat A\cond B$.}\label{fig2}
\end{figure}

For example, we have also considered the satisfaction relation of a conditional $A\cond B$ based on the rectangle in Fig.~\ref{fig2}, which extends the one based on triangle  illustrated in Fig.~\ref{fig_1}: namely, $\I \sat A\cond B$ iff  there is a ``rectangle'' $A \xrightarrow{f}B'\leq B\xrightarrow{g} A' \geq A$.  Now, in case $\leq$ is  transitive, as it is if $A\leq B$ is interpreted as $\models A \limp B$, such a configuration imposes the closure under the following property (proof omitted):
\begin{equation*}
\frac{ A\cond B, \ \ C\cond D, \ \ A\leq C, \ \ B\leq D}{  A\cond D }    
\end{equation*}
\nd Such a property may be counter-intuitive as it imposes implicitly a form of restricted (Mon) and (RW) that is not always desired. However such a reasoning pattern may become interesting if, for example, we interpret $A\leq C$ as stating that $C$ is \emph{similar} to $A$ instead: from $A\cond B$, $C\cond D$, $C$ similar to $A$, and $D$ similar to $B$ we derive $A\cond D$.
This kind of reinterpretation of $\leq$ would allow the analysis of totally different kinds of reasoning, depending on the meaning of $\leq$ and its properties, such as \eg~reflexivity, constrained forms of transitivity or symmetry.

We are looking forward to investigate entailment procedures and interpretation variants of $\leq$ in more detail.

\section{Conclusions}\label{concl}

There have been a few attempts to formalise non-classical forms of conditional reasoning that do not satisfy properties, like (RW), that are endemic in the possible-worlds semantics, \eg~\cite{CasiniEtAl2019,Rott2019}. The approach we consider here is quite different from that usually found in the literature, as our semantics  renounces the use of possible worlds: reasoning is  modelled through the manipulation of the choice functions $f$ and $g$, which we believe, is more flexible than the possible-worlds approach. Clearly, if we consider forms of reasoning that satisfy at least (LLE), (RLE), and (Anti-RW), we may revert also to the possible-worlds framework as presented in~\cite{CasiniEtAl2019}. The relationship between that semantics and the present one still needs to be investigated, however. Beside, let us note that another system, a deontic one, that satisfies implicitly (only) (RLE) has been presented by Parent and van der Torre \cite{XavierVandertorre2014}, and is based on the semantics of I/O logics \cite{Makinson2000}.

In summary, in this preliminary work, we have only started to investigate conditionals $A \cond B$ via the manipulation of the set-valued functions $f$ (the relevant effects of $A$) and $g$ (relevant conditions for $B$). 
Moreover, as mentioned in Section~\ref{sect_ent}, we think that by modifying the interpretation and the properties of $\leq$ the present semantics also paves the way to accommodate and analyse various other different kinds of non-classical reasoning.


\section*{Acknowledgments}
\nd This research was  partially supported by TAILOR (Foundations of Trustworthy AI – Integrating Reasoning, Learning and Optimization), a project funded by EU Horizon 2020 research and innovation programme under GA No 952215. 

\bibliographystyle{abbrv}
\bibliography{references}

\end{document}